\documentclass[runningheads]{llncs}
\usepackage[T1]{fontenc}
\usepackage{graphicx,verbatim}
\usepackage{amsfonts}
\usepackage{amsmath}
\usepackage[table,xcdraw]{xcolor} 
\usepackage{colortbl} 
\usepackage{xcolor}
\usepackage{marvosym}
\usepackage{hyperref}
\usepackage{amssymb}
\usepackage{graphicx,verbatim}
\usepackage{graphicx,verbatim}
\usepackage{multirow}
\definecolor{electricindigo}{rgb}{0.44, 0.0, 1.0}
\definecolor{lightblue}{RGB}{240,245,255}
\definecolor{darkblue}{RGB}{40,40,85}
\definecolor{babyblue}{rgb}{0.54, 0.81, 0.94}
\definecolor{pearDark}{HTML}{2980B9}
\definecolor{pearDarker}{HTML}{1D2DEC}
\usepackage[capitalize]{cleveref}
\crefname{section}{Sec.}{Secs.}
\Crefname{section}{Section}{Sections}
\Crefname{table}{Table}{Tables}
\crefname{table}{Tab.}{Tabs.}
\usepackage{fontawesome5}
\usepackage[T1]{fontenc}
\definecolor{electricindigo}{rgb}{0.44, 0.0, 1.0}
\definecolor{lightblue}{RGB}{240,245,255}
\definecolor{darkblue}{RGB}{40,40,85}
\definecolor{babyblue}{rgb}{0.54, 0.81, 0.94}
\definecolor{pearDark}{HTML}{2980B9}
\definecolor{pearDarker}{HTML}{1D2DEC}
\usepackage[capitalize]{cleveref}
\crefname{section}{Sec.}{Secs.}
\Crefname{section}{Section}{Sections}
\Crefname{table}{Table}{Tables}
\crefname{table}{Tab.}{Tabs.}
\usepackage{fontawesome5}

\usepackage[utf8]{inputenc} 
\usepackage[T1]{fontenc}    
\usepackage{url}            
\usepackage{booktabs}       
\usepackage{amsfonts}       
\usepackage{nicefrac}       
\usepackage{microtype}      
\usepackage[table]{xcolor}         
\usepackage{xcolor}         
\usepackage{graphicx}
\usepackage{subfigure}
\usepackage{amsmath}
\usepackage{multirow}
\usepackage{multicol}
\usepackage{nicematrix}

\definecolor{deblue}{RGB}{11,132,147}
\definecolor{ocra}{RGB}{204, 119, 34}

\usepackage{tikz}
\newcommand{\fcircle}[2][red,fill=red]{\tikz[baseline=-0.5ex]\draw[#1,radius=#2] (0,0.03) circle ;}

\usepackage{multirow}
\usepackage{colortbl}
\usepackage{tabulary}
\usepackage{etoolbox}
\usepackage{tikz}
\usepackage{pifont}
\begin{document}
\title{ProSMA-UNet: Decoder Conditioning for Proximal-Sparse Skip Feature Selection}
\titlerunning{ProSMA-UNet}
%

\author
{
Chun-Wun Cheng\inst{1}
\and
Yanqi Cheng  \inst{1} 
\and
Peiyuan Jing  \inst{2,3} 
\and
Guang Yang  \inst{3} 
\and
Javier A. Montoya-Zegarra \inst{2,4}   
\and 
Carola-Bibiane Schönlieb  \inst{1}
\and
Angelica I. Aviles-Rivero\inst{5}\textsuperscript{(\Letter)}}
%
\authorrunning{CW Cheng et al.}
%
\institute{
Department of Applied Mathematics and Theoretical Physics, University of Cambridge, UK 
\and
School of Engineering, Zurich University of Applied Sciences, CH
\and
Bioengineering Department and ImperialX, Imperial College London, UK
\and
Lucerne University Teaching and Research Hospital, CH
\and
Yau Mathematical Sciences Center, Tsinghua University, China \{\email{aviles-rivero@tsinghua.edu.cn}\}
}
  
\maketitle              
\begin{abstract} Medical image segmentation commonly relies on U-shaped encoder–decoder architectures such as U-Net, where skip connections preserve fine spatial detail by injecting high-resolution encoder features into the decoder. However, these skip pathways also propagate low-level textures, background clutter, and acquisition noise, allowing irrelevant information to bypass deeper semantic filtering—an issue that is particularly detrimental in low-contrast clinical imaging. 
Although attention gates have been introduced to address this limitation, they typically produce dense sigmoid masks that softly reweight features rather than explicitly removing irrelevant activations.
We propose ProSMA-UNet (Proximal-Sparse Multi-Scale Attention U-Net), which reformulates skip gating as a decoder-conditioned sparse feature selection problem.  ProSMA constructs a multi-scale compatibility field using lightweight depthwise dilated convolutions to capture relevance across local and contextual scales, then enforces explicit sparsity via an $\ell_1$ proximal operator with learnable per-channel thresholds, yielding a closed-form soft-thresholding gate that can remove noisy responses. To further suppress semantically irrelevant channels, ProSMA incorporates decoder-conditioned channel gating driven by global decoder context. Extensive experiments on challenging 2D and 3D benchmarks demonstrate state-of-the-art performance, with particularly large gains ($\approx20$\%) on difficult 3D segmentation tasks. Project page: \hyperlink{https://math-ml-x.github.io/ProSMA-UNet/}{https://math-ml-x.github.io/ProSMA-UNet/}


\end{abstract}

\section{Introduction}
Medical image segmentation is central to computer-aided diagnosis and treatment planning~\cite{lai2015deep}.  Deep learning has become the dominant framework for segmentation, and U-Net has become the prevailing backbone in many clinical segmentation pipelines~\cite{ronneberger2015u,isensee2021nnu,litjens2017survey}. Their effectiveness largely stems from skip connections that pass high-resolution encoder features to the decoder; however, this direct pathway can also bypass deeper semantic processing and propagate low-level textures, background clutter, and acquisition noise—especially in low-contrast imaging—leading to spurious regions and imprecise boundaries and motivating skip-aware designs that redesign or gate skip fusion to suppress irrelevant responses~\cite{zhou2020unetpp,oktay2018attention}.

Since U-Net popularized encoder--decoder segmentation with skip connections~\cite{ronneberger2015u}, numerous variants have been developed to improve representation, multi-scale fusion, and robustness by optimizing the encoder or decoder features~\cite{cheng2023continuous,cheng2025implicit,jiang2025rwkv}.
For volumetric segmentation, 3D U-Net and V-Net extend U-shaped designs with 3D operations~\cite{cciccek20163d,milletari2016v}. 
Other works focus on skip fusion, e.g., UNet++ with nested/dense skips and deep supervision~\cite{zhou2020unetpp} and UNet~3+ with full-scale aggregation~\cite{huang2020unet}. 
Attention U-Net introduces attention gate to reweight features~\cite{oktay2018attention}, nnU-Net highlights the strength of well-tuned U-Net pipelines~\cite{isensee2021nnu}, and Transformer-based hybrids (e.g., TransUNet, Swin-Unet) add global context while preserving skip-based localization~\cite{chen2024transunet,cao2022swin}. 
However, most skip-connection strategies still rely on aggregation or \emph{dense} reweighting (typically sigmoid masks), which attenuates but rarely removes incompatible activations, motivating \emph{decoder-conditioned sparse selection}; accordingly, ProSMA-UNet performs explicit skip feature selection via proximal sparsification rather than dense weighting. 
\\ \indent In this work, we propose \textbf{ProSMA-UNet} (\textbf{Pro}ximal-\textbf{S}parse
\textbf{M}ulti-scale \textbf{A}ttention U-Net), which reformulates skip connection as a decoder-conditioned sparse feature selection operator. Instead of directly mapping compatibility signals into dense masks, ProSMA-UNet first constructs a multi-scale compatibility field that measures the relevance of encoder features to the current decoder state across both local and contextual scales. This compatibility field is computed efficiently using lightweight depthwise dilated convolutions, enabling the gate to capture fine details while remaining sensitive to broader anatomical context. Crucially, ProSMA then enforces explicit sparsity via an $\ell_1$ proximal operation with learnable per-channel thresholds. This yields a closed-form soft-thresholding rule that can set incompatible activations to exact zeros, thereby
removing noisy responses rather than merely shrinking them. To further suppress semantically irrelevant information, we introduce a decoder-conditioned channel gating mechanism driven by global decoder context, which downweights skip channels that are inconsistent with the target structure at the current decoding stage. Together, these components turn skip connections into a principled, context-aware selection operator that retains useful high-resolution detail while preventing irrelevant signals from bypassing the decoder. 
\\ \indent \textbf{Our contributions are summarized as follows:} \fcircle[fill=deblue]{2pt} We show that skip connections in U-shaped medical segmentation networks act as a major pathway for noise and background leakage, and we formalise skip regulation as a decoder-conditioned sparse feature selection problem, rather than dense reweighting.
\fcircle[fill=deblue]{2pt} We propose ProSMA, a novel proximal-sparse skip gating mechanism that constructs a multi-scale decoder–encoder compatibility field and enforces explicit sparsity via a learnable $\ell_1$ proximal operator, yielding closed-form soft-thresholding and exact removal of irrelevant skip activations.
\fcircle[fill=deblue]{2pt} We introduce a decoder-conditioned channel gating mechanism that complements spatial sparsification by suppressing semantically inconsistent feature channels. --- We provide a theoretical analysis proving that proximal sparse skip gating achieves exact feature selection and is non-expansive, guaranteeing robustness to noise in the compatibility field. \fcircle[fill=deblue]{2pt} Extensive experiments on challenging 2D and 3D benchmarks demonstrate state-of-the-art performance, with particularly large gains ($\approx20$\%) on difficult 3D segmentation tasks.

\section{Methodology}

\subsection{Problem Statement and Overall architecture}

Let $x_s \in \mathbb{R}^{H_s \times W_s \times C_s}$ denote the encoder feature map at scale $s$, and let
$g_{s+1} \in \mathbb{R}^{H_s \times W_s \times C'_s}$ denote the decoder feature map upsampled to the same spatial resolution.
Conventional U-shaped skip connections pass $x_s$ directly to the decoder to recover high-frequency detail, but $x_s$ may also carry clutter, noise, and semantically irrelevant structures.
We therefore model the skip pathway as a decoder-conditioned selection operator
$\mathcal{G}_s : (x_s, g_{s+1}) \mapsto \tilde{x}_s$,
where $\tilde{x}_s$ retains only those components of $x_s$ that are relevant under the context $g_{s+1}$.

\begin{figure}[t!]
\centering
\includegraphics[width=0.95\textwidth]{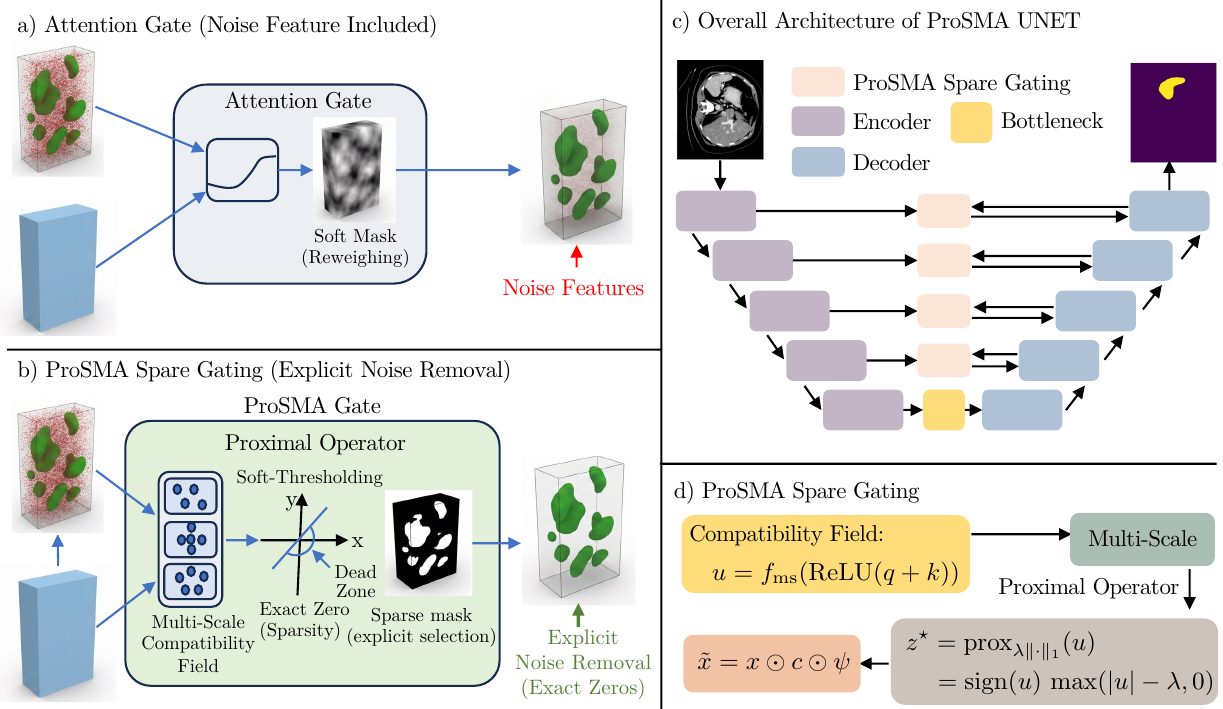}
\caption{ProSMA-UNet motivation and overview. 
(a) Conventional attention gates generate a \emph{dense} soft mask (sigmoid reweighting), which can still pass weak but harmful skip activations (noise features) into the decoder. 
(b) Our ProSMA sparse gating constructs a multi-scale decoder--encoder compatibility field and applies an $\ell_1$ proximal (soft-thresholding) operator to induce \emph{explicit sparsity} (exact zeros), enabling direct removal of irrelevant skip responses. 
(c) ProSMA-UNet integrates the proposed sparse gating module at each skip connection to condition high-resolution feature transfer on decoder context. 
(d) Overview of ProSMA Spare Gating.}
\label{fig:prosma_overview}
\end{figure}
ProSMA-UNet follows a standard U-shaped encoder--decoder architecture. The encoder consists of five resolution levels implemented via max-pooling and residual feature extraction blocks, producing feature maps
$\{x_1, x_2, x_3, x_4, x_5\}$, where $x_1$ has the highest spatial resolution and $x_5$ corresponds to the bottleneck representation.
The bottleneck is implemented as a residual block operating on $x_5$.
The decoder mirrors the encoder with four upsampling stages using bilinear interpolation followed by convolution, yielding decoder features
$\{g_5, g_4, g_3, g_2\}$.

At each decoding stage $s$, the encoder feature $x_s$ is first filtered by the proposed ProSMA gating operator $\mathcal{G}_s$ using the decoder context $g_{s+1}$.
The resulting feature $\tilde{x}_s$ is concatenated with the upsampled decoder feature and refined by a residual block:
$g_s
=
\mathrm{ResBlock}
\Big(
\mathrm{Concat}
\big(
\mathcal{G}_s(x_s, g_{s+1}),
\;
\mathrm{Up}(g_{s+1})
\big)
\Big)$. Figure \ref{fig:prosma_overview} (c) and (d) provide an overview of our method and the proposed Proximal-Sparse Gate.

\subsection{Proximal-Sparse Multi-Scale Attention Gate}
In medical image segmentation, only a small subset of encoder features are relevant to the target anatomy, while many skip features correspond to background clutter or acquisition noise.
This suggests that skip gating is fundamentally a \emph{feature selection} problem rather than a soft reweighting problem. Figure \ref{fig:prosma_overview} (a) shows that the attention mask cannot filter out the noise, while our method can remove the noise by Proximal-Sparse Gate. Accordingly, we reinterpret skip connections as a sparse, context-dependent selection operator.
Given an encoder feature map $x \in \mathbb{R}^{H \times W \times C_x}$ and a decoder feature map
$g \in \mathbb{R}^{H \times W \times C_g}$ at the same spatial resolution, we first project them into a shared latent space:
\begin{equation}
q = W_x x,
\qquad
k = W_g g,
\label{eq:proj}
\end{equation}
where $W_x$ and $W_g$ are $1 \times 1$ convolutions producing latent features in
$\mathbb{R}^{H \times W \times C_a}$, where $C_a$ denotes the dimensionality of a latent compatibility space used for sparse feature selection.

 \textbf{Multi-Scale Compatibility Field}.
To capture both local and contextual relevance, we construct a multi-scale compatibility field
\begin{equation}
u
=
f_{\mathrm{ms}}\!\left(
\mathrm{ReLU}(q + k)
\right),
\qquad
u \in \mathbb{R}^{H \times W \times C_a},
\label{eq:compat}
\end{equation}
where the ReLU nonlinearity enforces non-negative compatibility responses.
The operator $f_{\mathrm{ms}}$ aggregates information across multiple receptive fields and is defined as
$f_{\mathrm{ms}}(v)
=
\phi\!\left(
\left[
D^{(d_1)} v,\;
D^{(d_2)} v,\;
\dots,\;
D^{(d_m)} v
\right]
\right)$.
Here $D^{(d)}$ denotes a depthwise convolution with dilation rate $d$, and $\phi$ is a $1 \times 1$ convolution that fuses information across scales.
The use of depthwise convolutions preserves channel-wise independence, which is consistent with the subsequent per-channel sparse selection.
By ~\eqref{eq:proj} and~\eqref{eq:compat}, the proposed gating module defines a multi-scale compatibility field $u(x,g) \in \mathbb{R}^{H \times W \times C_a}$, whose entries quantify the alignment between encoder features $x$ and decoder context $g$ across spatial scales. In practice, only a subset of these compatibility responses is expected to be informative, while the remainder correspond to background clutter or acquisition noise.

From a variational perspective, this motivates imposing sparsity as a structural prior on the compatibility field.
Specifically, skip gating is formulated as the problem of suppressing weak or uninformative compatibility responses while preserving salient ones.
This interpretation treats $u$ as an unregularised signal to be filtered and does not require it to arise as the gradient of an explicit energy functional.
Under this formulation, sparsity is enforced directly on the compatibility responses, leading naturally to a proximal optimisation framework.

\noindent\textbf{Proximal Sparse Gating}. Rather than directly mapping $u$ to an attention mask, we enforce explicit sparsity by solving the following optimisation problem:
\begin{equation}
z^{\star}
=
\arg\min_{z}
\;
\frac{1}{2}\|z - u\|_2^2
+
\lambda \|z\|_1,
\label{eq:l1_prox}
\end{equation}
where $\lambda > 0 $ controls the sparsity level. This formulation corresponds to a variational regularisation of the compatibility field under an $\ell_1$ prior and is well-defined regardless of how $u$ is generated.

Our proximal sparse gating corresponds to one learned proximal gradient step under $\ell_1$ regularisation. Equation~\ref{eq:l1_prox} admits a closed-form solution given by the proximal operator of the $\ell_1$ norm:
\begin{equation}
z^{\star}
=
\mathrm{prox}_{\lambda \|\cdot\|_1}(u)
=
\mathrm{sign}(u)\,\max(|u| - \lambda, 0).
\end{equation}
which yields exact zero-valued coefficients and therefore induces explicit sparsity, implementing feature selection rather than continuous reweighting.
To allow adaptive sparsity, we parameterise $\lambda$ as a learnable, per-channel threshold:
$\lambda = \mathrm{softplus}(\theta),
\qquad
\theta \in \mathbb{R}^{C_a}$,
where $C_a$ denotes the dimensionality of the latent compatibility space.

\noindent{\textbf{Spatial Attention Mask.}}
The sparse latent field is converted into a spatial attention mask via a pointwise projection:
$\psi
=
\sigma\!\left(
\Psi(z^{\star})
\right),
\qquad
\psi \in (0,1)^{H \times W}$,
where $\Psi$ is a $1 \times 1$ convolution that aggregates latent compatibility channels and $\sigma$ denotes the sigmoid function.

\noindent{\textbf{Decoder-Conditioned Channel Gating.}} Spatial sparsity alone does not suppress semantically irrelevant feature channels.
We therefore introduce decoder-conditioned channel gating: $c
=
\sigma\!\left(
\mathrm{MLP}\!\left(
\mathrm{GAP}(g)
\right)
\right),
\qquad
c \in (0,1)^{C_x}$,
where $\mathrm{GAP}$ denotes global average pooling and $\mathrm{MLP}$ is a lightweight multilayer perceptron.

Considering previous  components, the proposed ProSMA skip operator can be written compactly as
\begin{equation}
\tilde{x} = x \odot c \odot \psi,
\end{equation}
where $\psi$ decides \emph{where} to pass information (spatial selectivity) and  $c$ decides \emph{what} semantic channels to emphasize (channel selectivity).  This formulation transforms skip connections from unconstrained soft attention into a principled, context-aware sparse selection operator.

\subsection{Theoretical Foundation for ProSMA}
We analyse the theoretical properties of the proposed proximal sparse gating operator.
ProSMA enforces sparsity by applying an $\ell_1$ proximal step to the multi-scale compatibility field $u$, producing a filtered response $z$ that controls which skip features are propagated to the decoder.
We focus on two desirable properties for medical segmentation: (i) \emph{exact feature selection} (irrelevant activations are set to zero) and (ii) \emph{stability under perturbations} (noise in $u$ is not amplified).
The following theorem formalises these guarantees for our proximal sparse gating.
\begin{theorem}[Exact Feature Selection and Stability of Proximal Sparse Gating]
\label{thm:prox_sparse_properties}
Let $u \in \mathbb{R}^{H \times W \times C_a}$ be the multi-scale compatibility
field defined in Eq, and let $\lambda \in \mathbb{R}_{+}^{C_a}$ be the
per-channel sparsity threshold defined in~\eqref{eq:l1_prox}. Consider the proximal sparse
gating problem
$z
=
\arg\min_{z \in \mathbb{R}^{H \times W \times C_a}}
\;
\frac{1}{2}\|z - u\|_F^2
+
\sum_{c=1}^{C_a} \lambda_c \|z_{:,:,c}\|_1$.
Then the following properties hold: \textbf{Monotonic sparsity with respect to $\lambda$.} If $\lambda'_c \ge \lambda_c$ for all $c$, then
$\mathrm{supp}(z(u;\lambda'))
    \subseteq
    \mathrm{supp}(z(u;\lambda))$.
Thus, increasing the learned sparsity parameter can only remove
active features and never introduce new ones.
\textbf{Non-expansive stability.}
For any $u,v \in \mathbb{R}^{H \times W \times C_a}$,
$\| z(u;\lambda) - z(v;\lambda) \|_F
    \le
    \| u - v \|_F$ .
Hence, proximal sparse gating is $1$-Lipschitz and cannot amplify
perturbations in the compatibility field.
\end{theorem}
\begin{proof}
\textbf{Monotonic sparsity.}
The proximal $\ell_1$ operator is the soft-thresholding map applied
entrywise:$z_{p,c}
=
\operatorname{sign}(u_{p,c})
\max(|u_{p,c}|-\lambda_c,0)$.
If $\lambda'_c \ge \lambda_c$, then
$|u_{p,c}| \le \lambda_c
\;\Rightarrow\;
|u_{p,c}| \le \lambda'_c$.
Therefore any entry that is zero under $\lambda$ remains zero under
$\lambda'$, implying
$\mathrm{supp}(z(u;\lambda'))
\subseteq
\mathrm{supp}(z(u;\lambda))$. 

\textbf{Non-expansiveness.}
Define the scalar soft-threshold operator
$S_\lambda(t)
=
\operatorname{sign}(t)\max(|t|-\lambda,0)$.
This function is piecewise linear with slopes in $[0,1]$.
Hence for any $a,b \in \mathbb{R}$,
$|S_\lambda(a)-S_\lambda(b)| \le |a-b|$.
Applying this entrywise and summing over all spatial positions and channels yields $\| z(u;\lambda) - z(v;\lambda) \|_F^2
=
\sum_{p,c}
|S_{\lambda_c}(u_{p,c}) - S_{\lambda_c}(v_{p,c})|^2
\le
\|u-v\|_F^2$. Taking square roots gives the desired result.
\end{proof}

This theorem formalises that ProSMA gating is feature selection (exact zeros in $z$) rather than dense sigmoid reweighting, and it guarantees that the proximal step is stable (non-expansive), so it cannot amplify noise in the multi-scale compatibility field before producing the final spatial mask.

\section{Experimental Results}
\textbf{Dataset, Baselines and Experiment Setting.}
We evaluated ProSMA-UNet on three 2D and two 3D medical image segmentation benchmarks. The 2D datasets include BUSI~\cite{al2020dataset}, GlaS~\cite{valanarasu2021medical} and Kvasir-SEG~\cite{jha2020kvasir}. BUSI contains 647 benign and malignant breast ultrasound images resized to 256 × 256. GlaS consists of 612 frames (384 × 288) from 31 sequences; following prior work, we used 165 images resized to 512 × 512. Kvasir-SEG includes 1,000 colonoscopy images (800 training, 200 testing) at 256 × 256 for polyp segmentation.
For 3D segmentation, we used the Spleen and Colon datasets from the Medical Segmentation Decathlon ~\cite{antonelli2022medical}. The Spleen dataset comprises 61 abdominal CT volumes with notable inter-patient organ size variation. The Colon dataset contains 190 abdominal CT volumes for tumor segmentation, characterized by high heterogeneity in tumor appearance and shape, making robust 3D segmentation particularly challenging.

We compared to the representative baselines, including U-Net~\cite{ronneberger2015u}, U-Net++~\cite{zhou2020unetpp}, and Implicit U-KAN 2.0~\cite{cheng2025implicit} with a standard skip connection that concatenates the two features.  Then U-NeXt~\cite{valanarasu2022unext}, Rolling-UNet~\cite{liu2024rolling}, and UKAN~\cite{li2025u} use sum-element-wise addition skip connections. Finally, Attention UNet~\cite{oktay2018attention} uses an attention gate to reweight the features. For BUSI and Glas, we follow the same experiment setting in UKAN while the Kvasir SEG follow same setting from
Implicit U-KAN 2.0. All the experiments are run on an NVIDIA A100 40GB.

\begin{table}[t!]
\renewcommand{\arraystretch}{1.25}
\caption{
Quantitative comparison on three heterogeneous medical imaging benchmarks. Results are reported over three independent runs, demonstrating both superior performance and stable convergence. The best results are indicated by
\textbf{\colorbox{green!20}{Green}}.
}
\label{tab:results-multilevel}
\begin{NiceTabular}{
l | l|
>{\centering\arraybackslash}p{1.6cm}
>{\centering\arraybackslash}p{1.6cm} |
>{\centering\arraybackslash}p{1.6cm}
>{\centering\arraybackslash}p{1.6cm} |
>{\centering\arraybackslash}p{1.6cm}
>{\centering\arraybackslash}p{1.6cm}
}
\hline
\rowcolor[HTML]{EFEFEF}
\multirow{2}{*}{Method} &
\multirow{2}{*}{Year} &
\multicolumn{2}{c|}{BUSI} &
\multicolumn{2}{c|}{GlaS} &
\multicolumn{2}{c}{Kvasir-SEG} \\
\cline{3-8}
\rowcolor[HTML]{EFEFEF}
& & IoU $\uparrow$ & F1 $\uparrow$
& IoU $\uparrow$ & F1 $\uparrow$
& IoU $\uparrow$ & F1 $\uparrow$ \\
\hline

U-Net & 2015 &
57.22$\pm$4.74 & 71.91$\pm$3.54 &
86.66$\pm$0.91 & 92.79$\pm$0.56 &
54.58$\pm$0.18 & 79.59$\pm$0.08  \\

Att-UNet & 2018 &
55.18$\pm$3.61 & 70.22$\pm$2.88 &
86.84$\pm$1.19 & 92.89$\pm$0.65 &
62.75$\pm$0.56 & 83.36$\pm$0.78 \\

U-Net++ & 2018 &
57.41$\pm$4.77 & 72.11$\pm$3.90 &
87.07$\pm$0.76 & 92.96$\pm$0.44 & 
61.93$\pm$0.74 & 83.33$\pm$0.42  \\

U-NeXt & 2022 &
59.06$\pm$1.03 & 73.08$\pm$1.32 &
84.51$\pm$0.37 & 91.55$\pm$0.23 &
49.50$\pm$0.48 & 76.64$\pm$0.26  \\

R-UNet & 2024 &
61.00$\pm$0.64 & 74.67$\pm$1.24 &
86.42$\pm$0.96 & 92.63$\pm$0.62 &
60.28$\pm$1.03 & 82.25$\pm$0.35  \\

UKAN & 2024 &
63.38$\pm$2.83 & 76.40$\pm$2.90 &
87.64$\pm$0.32 & 93.37$\pm$0.16 &
53.78$\pm$2.86 & 78.58$\pm$1.41 \\

UKAN2.0 & 2025 &
63.05$\pm$1.01 & 76.86$\pm$0.66 & -- & --  & 58.89$\pm$1.38 & 81.04$\pm$0.65 \\
\hline
P-UNET & 2026 &
{\cellcolor[HTML]{D7FFD7}66.24$\pm$1.32} &
{\cellcolor[HTML]{D7FFD7}79.29$\pm$0.93} &
{\cellcolor[HTML]{D7FFD7}88.76$\pm$0.14} &
{\cellcolor[HTML]{D7FFD7}94.04$\pm$0.08} &
{\cellcolor[HTML]{D7FFD7}66.23$\pm$2.15} &
{\cellcolor[HTML]{D7FFD7}85.43$\pm$1.19} \\
\hline
\end{NiceTabular}
\vspace{-0.3cm}
\label{tab:exp_seg}
\end{table}

\fcircle[fill=deblue]{2pt} \textbf{2D segmentation} Table~\ref{tab:exp_seg} reports a quantitative comparison between our ProSMA-UNet (P-UNET) and other baselines \footnote{
For BUSI and GlaS, baseline results are taken from U-KAN~\cite{li2025u} under the same protocol. UKAN2.0 could not be evaluated on GlaS due to GPU memory limits.
} on three heterogeneous medical scenarios.  As shown, P-UNET consistently achieves the best performance across all datasets, establishing a clear new state of the art performance. On BUSI, P-UNET outperforms the strongest competing results by +2.86 IoU and +2.43 F1 (+4.51\% and +3.16\% relative gains, respectively). On GlaS, P-UNET achieves 88.76$\pm$0.14 IoU and 94.04$\pm$0.08 F1, exceeding the best baseline by +1.12 IoU and +0.67 F1. More importantly, on the challenging Kvasir-SEG benchmark, P-UNET delivers the largest improvement, obtaining 66.23$\pm$2.15 IoU and 85.43$\pm$1.19 F1, i.e., +3.48 IoU and +2.07 F1 over the best competing method, and a substantial +12.45 IoU / +6.85 F1 gain over U-KAN. These consistent and margin-wise improvements across ultrasound, histology, and endoscopy images verify the strong generalization ability of P-UNET. 

\begin{table}[t!]
\centering
\begin{minipage}{0.6\textwidth}
\centering
\caption{Comparison of 3D segmentation on Spleen and Colon datasets on the F1.}
\label{tab:3D_combined}
\begin{tabular}{>{\centering\arraybackslash}p{3cm} |>{\centering\arraybackslash}p{2cm} |>{\centering\arraybackslash}p{2cm}}
\hline
\cellcolor[HTML]{EFEFEF}\textsc{Method} &
\cellcolor[HTML]{EFEFEF}Spleen &
\cellcolor[HTML]{EFEFEF}Colon \\ \hline
U-Net 3D     & 90.99 $\pm$ 1.31 & 51.79 $\pm$ 2.35 \\
U-KAN 3D     & 95.02 $\pm$ 0.28 & 51.32 $\pm$ 0.19 \\
UKAN2.0 3D   & 97.02 $\pm$ 0.16 & 53.05 $\pm$ 0.29 \\
Ours         & \cellcolor[HTML]{D9FFD9}\textbf{97.59 $\pm$ 0.10} 
             & \cellcolor[HTML]{D9FFD9}\textbf{63.14 $\pm$ 0.99} \\ \hline
\end{tabular}
\end{minipage}
\hfill
\begin{minipage}{0.38\textwidth}
\centering
\caption{Ablation study on gating mechanisms (F1).}
\label{tab:ablation}
\begin{tabular}{>{\centering\arraybackslash}p{2.8cm} |>{\centering\arraybackslash}p{1.4cm}}
\hline
\cellcolor[HTML]{EFEFEF}\textsc{Method} &
\cellcolor[HTML]{EFEFEF}F1 \\ \hline
Without PSG & 74.05 \\
Only SS & 77.47 \\
Only CG   & 75.22 \\
Full Model            & \cellcolor[HTML]{D9FFD9}\textbf{80.13} \\ \hline
\end{tabular}
\end{minipage}
\vspace{-0.4cm}
\end{table}

\fcircle[fill=deblue]{2pt} \textbf{3D segmentation.}
To assess scalability beyond 2D, we extend our model to a 3D variant and evaluate it on two volumetric tasks (Spleen and Colon), comparing against 3D U-Net and KAN-based baselines (U-KAN 3D and UKAN2.0 3D). As shown in Table~\ref{tab:3D_combined}, our method achieves the best results on both datasets: 97.59$\pm$0.10 on Spleen, improving over the strongest baseline (UKAN2.0 3D) by +0.57, and 63.14$\pm$0.99 on the more challenging Colon benchmark, surpassing UKAN2.0 3D by +10.09 (about +19.0\% relative). These gains indicate that proximal-sparse gate generalizes effectively to 3D and yields stronger, more reliable volumetric representations across different anatomical targets.

\fcircle[fill=deblue]{2pt} \textbf{Visualisation. }
Figure \ref{fig:visualisation} presents qualitative visualisations supporting our quantitative gains and highlighting our model's superior segmentation across representative cases. On the 2D GlaS dataset, P-UNET generates uniform, coherent gland masks with smooth, well-defined boundaries, in contrast to the irregular and fluctuating edges produced by U-Net, Att-UNet, and U-KAN. For BUSI ultrasound images, our model accurately localises the core tumour region, while U-Net and U-KAN capture only partial areas and Att-UNet is distracted by speckle noise, yielding a dispersed prediction. In the 3D Spleen task, all methods recover the overall shape, but P-UNET produces markedly smoother surfaces; error maps confirm fewer and less pronounced errors than U-KAN and UKAN2.0.
\\ \fcircle[fill=deblue]{2pt} \textbf{Ablation Studies. }
We ablate the proposed skip-gating components in Table~\ref{tab:ablation}. Removing Proximal Sparse Gating (PSG) causes a sharp performance drop, indicating that ungated skip connections pass irrelevant or noisy activations that hinder segmentation. Using Sparse Spatial (SS) selection alone yields a substantial improvement, suggesting it suppresses background and ambiguous regions in skip features. Channel gating (CG) alone provides a smaller gain, showing it is helpful but insufficient to fully filter harmful skip information. Combining SS and CG achieves the best result, improving by +6.08 F1 over the ungated baseline. Overall, SS and CG are complementary: SS controls where to pass information, while CG determines which semantic channels to emphasize, enabling noise removal and stronger medical image segmentation.

\begin{figure}[t!]
\centering
\includegraphics[width=1\textwidth]{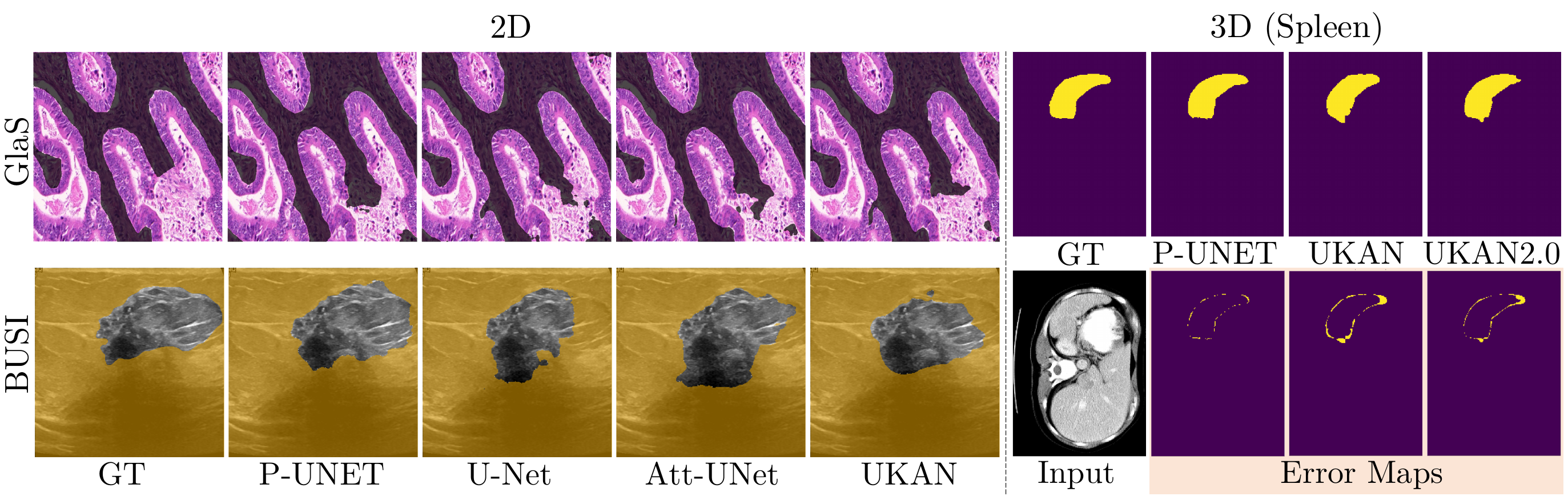}
\caption{Qualitative comparison of segmentation masks produced by P-UNET and baseline methods across 2D and 3D medical imaging datasets.
}
\label{fig:visualisation}
\end{figure}

\section{Conclusion}
In this work, we propose \textbf{ProSMA-UNet}, which treats skip connections as a \emph{decoder-conditioned feature selection} problem rather than unconstrained feature transfer. ProSMA builds a multi-scale decoder--encoder compatibility field using efficient depthwise dilated convolutions, then enforces \emph{explicit sparsity} with a learnable $\ell_1$ proximal (soft-thresholding) operator to remove irrelevant skip activations. A decoder-conditioned channel gate further suppresses semantically inconsistent channels. We also show that proximal sparse gating provides exact feature selection and is non-expansive, ensuring stability under perturbations. Experiments on diverse 2D and 3D benchmarks consistently outperform strong U-Net variants, with the largest gains on challenging 3D tasks.

\subsubsection{Acknowledgments} 
CWC and JAMZ are supported by the Swiss National Science Foundation under grant number 20HW-1 220785. YC is funded by an AstraZeneca studentship and a Google studentship. CBS  acknowledges support from the Philip Leverhulme Prize, the Royal Society Wolfson Fellowship, the EPSRC grants EP/S026045/1, EP/T003553/1, EP/N014588/1, the Wellcome Innovator Award RG98755 and the Alan Turing Institute.
AIAR gratefully acknowledges the support of the Yau Mathematical Sciences Center, Tsinghua University. This work is also supported by the Tsinghua University Dushi Program.

\bibliographystyle{splncs04}
\bibliography{main}
\end{document}